\documentclass[reqno,11pt]{article}

\usepackage{authblk}
\usepackage{color}
\usepackage{hyperref}
\usepackage{amsmath}
\usepackage{amsfonts}
\usepackage{amssymb}
\usepackage{graphicx}
\usepackage{amsthm}
\usepackage{natbib}
\usepackage {tikz}
\usetikzlibrary {positioning}
\usetikzlibrary{decorations.shapes}
\definecolor {processblue}{cmyk}{0.96,0,0,0}

\newtheorem{theorem}{Theorem}[section]
\newtheorem{lemma}[theorem]{Lemma}

\newtheorem{remark}[theorem]{Remark}

\DeclareMathOperator{\KL}{KL}

\title{A Deep Learning Approach to Unsupervised Ensemble Learning}
\author[1]{Uri Shaham}
\author[2]{Xiuyuan Cheng}
\author[3]{Omer Dror}
\author[3]{Ariel Jaffe}
\author[3]{Boaz Nadler}
\author[1]{Joseph Chang}
\author[4]{Yuval Kluger}
\affil[1]{Department of Statistics, Yale University}
\affil[2]{Program of Applied Mathematics, Yale university}
\affil[3]{Faculty of Mathematics and Computer Science, Weizmann Institute}
\affil[4]{Department of Pathology, Yale University}

\date{}                                           

\begin{document}
\maketitle

\begin{abstract} 
We show how deep learning methods can be applied in the context of crowdsourcing and unsupervised ensemble learning. 
First, we prove that the popular model of Dawid and Skene, which assumes that all classifiers are
conditionally independent, is {\em equivalent} to a Restricted Boltzmann Machine (RBM) with a single hidden node. 
Hence, under this model, the posterior probabilities of the true labels can be instead estimated via a trained RBM.
Next, to address the more general case, where classifiers may strongly violate the conditional independence assumption, 
we propose to apply RBM-based Deep Neural Net (DNN).
Experimental results on various simulated and real-world datasets demonstrate that our proposed DNN approach 
outperforms other state-of-the-art methods, in particular when the data violates the conditional independence assumption. 

\end{abstract} 

\section {Introduction}
In recent years, crowdsourcing applications gained significant popularity, and consequently much academic attention. At the same time, deep learning has become a major tool in machine learning and artificial intelligence, demonstrating impressive performance in several applications, including computer vision, speech recognition and natural language processing.  

The goal of this paper is to show that deep learning methods can also be applied to the areas of crowdsourcing and unsupervised ensemble learning, and provide state-of-the-art results. 
In unsupervised ensemble learning, one is given the predictions of $d$ classifiers on a set of $n$ instances and the goal is to recover the true, unknown label of each instance.
\citet{dawid1979maximum} were among the first to consider such a setup. They assumed that the classifiers are conditionally independent given the true labels. 
We refer to this model as the {\em DS} model and also as the {\em Conditional Independence model}.

Despite its simplicity, computing the maximum likelihood estimates of the classifiers' accuracies and the true labels in the DS model is a non-convex optimization problem. In their paper, Dawid and Skene estimated these quantities by the EM algorithm, which is only guaranteed to converge to a local optimum. In recent years, several authors developed computationally efficient spectral methods that are asymptotically consistent under the DS\ model, see   \citet{zhang2014spectral, parisi2014ranking, jain2013learning}; \citet{jaffe2014estimating} and references therein. 

The model of Dawid and Skene relied on two key assumptions that typically do not  hold in practice:\ (i) that classifiers make perfectly independent errors; and (ii)\ that these errors are uniformly distributed across all instances. To address the second issue above, several authors proposed richer models, that include parameters such as instance difficulty and varying skills of annotators across different regions of the input space, see for example  \citet{raykar2010learning}, \citet{whitehill2009whose} and \citet{welinder2010multidimensional}. 

In contrast, relatively few works considered relaxations of the conditional independence assumption: \citet{platanios2014estimating} proposed to estimate the accuracies of possibly dependent classifiers, via their agreement rates over classifier groups of different sizes. 
\citet{donmez2010unsupervised} proposed a model with pairwise interactions between all classifiers. Closest to our approach is the work of \citet{jaffe2015unsupervised}, who assumed that some of the classifiers may be conditionally dependent, yet their dependency structure can be accurately described by a tree of depth 2.

In this manuscript, we propose a deep learning approach to unsupervised ensemble learning problems with possibly dependent classifiers, where the conditional independence assumption is strongly violated. 
We make the following contributions. First, we show that the DS model has an equivalent parametrization in terms of a Restricted Boltzmann Machine (RBM) with a single hidden node. Hence, under this model, the posterior probability of the true labels can be estimated from a trained RBM. 
Next, to tackle violations of conditional independence, we show how a RBM-based Deep Neural Net (DNN) can be  applied to unsupervised ensemble learning, and propose a heuristic for determining the DNN architecture. 
Experimentally, we compare our approach to several state-of-the-art methods that are based on the conditional independence assumption and relaxations of it. We show that our DNN approach often performs better than the other methods on both simulated and real world datasets. Remarkably, we demonstrate that in some cases, while the raw representation of the data contains correlated features, the learned features in the last hidden layer are almost perfectly uncorrelated.

The structure of this manuscript is as follows: in Section~\ref{sec:probSetup} we give a formal definition of the problem. 
A brief background on RBMs is given in Section~\ref{sec:rbm}. 
In Section~\ref{sec:condIndCase} we show how RBMs can be used to predict the true labels, under the assumption of conditional independence. 
In Section~\ref{sec:RBM-DNN} we describe how to estimate the labels using a RBM-based DNN.
Experimental results are reported in Section~\ref{sec:experimentalResults}.
The manuscript concludes with a brief summary in Section~\ref{sec:conclusions}. 
Proofs appear in the appendix.


\subsection{Notation}
Throughout this manuscript, $X,H,Y$ are random variables, $p_\theta,p_\lambda$ are probability densities, parametrized by $\theta,\lambda$, respectively. We think of $p_\theta$ as the distribution generating the data and of $p_\lambda$ as the RBM model distribution.
When the context is clear, we occasionally write $p(x)$ as a shorthand for $p(X=x)$.
The dimensions of the input data and the sample size are denoted by $d$ and $n$, respectively.
We use $\sigma(\cdot)$ to denote the sigmoid function 
\begin{equation}
\sigma(z) = \frac{1}{1+e^{-z}}.\label{eq:sigmoid}
\end{equation}

\section{Problem Setup}
\label{sec:probSetup}
Let $X \in \{0,1 \}^d,\; Y \in \{0,1\}$ be random variables. We refer to $Y$ as the label of $X$. 
The pair $(X,Y)$ has a joint distribution, parametrized by $\theta$ and denoted by $p_\theta(X,Y)$, which  is given by 
\begin{equation}
p_\theta(X,Y) = p_\theta(Y)p_\theta(X|Y). \notag
\end{equation}
The joint distribution $p_\theta(X,Y)$ is not known to us, and neither are the marginals $p_\theta(X),p_\theta(Y)$.
Let $(x^{(1)},y^{(1)}),\ldots,(x^{(n)},y^{(n)})$ be $n$ i.i.d samples from $p_\theta(X,Y)$. 
In unsupervised ensemble learning, we observe $x^{(1)},\ldots,x^{(n)}$ and the learning task is to recover $y^{(1)},\ldots,y^{(n)}$. 
In this application, the binary vector $X = (X_1,\ldots,X_d)^T$ contains the predictions of $d$ classifiers or annotators on an instance, whose label $Y$ is unobserved. 


\subsection {The Conditional Independence Model}
In their seminal  paper, \citet{dawid1979maximum}, assumed that the conditional distribution $p_\theta(X|Y)$ factorizes, i.e.,
 \begin{equation}
p_\theta(X|Y) \equiv \prod_{i=1}^d p_\theta(X_i|Y). \label{eq:condInd}
 \end{equation}
Eq.~\eqref{eq:condInd}, also known as the {\em conditional independence model}, is depicted in Figure~\ref{fig:NB}.
\begin{figure}[t]
  \centering
  \begin {tikzpicture}[-latex ,auto ,node distance =4 cm and 5cm ,on grid ,
                semithick ,
                state/.style ={ circle ,top color =white , bottom color = processblue!20 ,
                        draw,processblue , text=blue , minimum width =1 cm}]
                \node[state] (Y) at (2,2) {$Y$};
                \node[state] (f1) at (-1,0) {$X_1$};
                \node[state] (fi) at (1.5,0) {$X_i$};
                \node[state] (fm) at (4,0) {$X_d$};
                \path (Y) edge node [above =0.15 cm,left = 0.15cm] {$\psi_1,\eta_1$} (f1);
                \path (Y) edge node {$\psi_i,\eta_i$} (fi);
                \path (Y) edge node {$\psi_d,\eta_d$} (fm);
        \end{tikzpicture}
  \caption{The conditional independence model, studied by \citet{dawid1979maximum}.}
  \label{fig:NB}
 \end{figure}
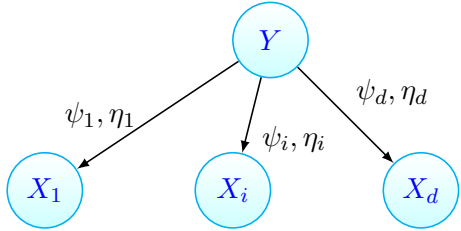
It is fully parametrized by $\theta = (\{\psi_i: i=1,...,d \}, \{\eta_i: i=1,...,d \}, \pi)$, where
\begin{align}
&\psi_i = \Pr(X_i=1|Y=1),\; \eta_i = \Pr(X_i=0|Y=0), \notag\\
&\pi = \Pr(Y=1). \notag
\end{align}
$\psi_i, \eta_i$ are often referred to as sensitivity and specificity, respectively.  Under the interpretation of the $X_i$'s being classifiers, the sensitivity and specificity quantify the competence of the classifiers or annotators and the conditional independence assumption means that all $d$ classifiers make independent errors.

The conditional independence model is often overly simplistic. In this manuscript we propose to apply deep learning techniques, specifically RBM-based DNNs, for unsupervised ensemble learning problems, where the conditional independence is not likely to hold. The following section gives essential background on RBMs, section~\ref{sec:condIndCase} shows that a RBM with a single hidden node is equivalent to the conditional independence model, and section~\ref{sec:RBM-DNN} presents our RBM-based DNN approach.


\section {Restricted Boltzmann Machines} \label{sec:rbm}
A Restricted Boltzmann Machine (RBM) is an undirected bipartite graphical model, consisting of a set $X$ of $d$ visible binary random variables and a set $H$ of $m$ hidden binary random variables, arranged in two layers, which are fully connected to each other. An illustration of a RBM is depicted in Figure~\ref{fig:RBM}.
\begin{figure}[ht!]
  \centering
  \begin {tikzpicture}[auto ,node distance =4 cm and 5cm ,on grid ,
                semithick ,
                state/.style ={ circle ,top color =white , bottom color = processblue!20 ,
                        draw,processblue , text=blue , minimum width =1 cm}]
          \node[state] (H1) at (8.5,3) {$H_1$};
                \node[state] (Hm) at (11.5,3) {$H_m$};
                \node[state] (X1) at (7,0) {$X_1$};
                \node[state] (Xi) at (10,0) {$X_i$};
                \node[state] (Xd) at (13,0) {$X_d$};
                \path (H1) edge node [above =0.15 cm,left = 0.15cm] {$w_{11}$}(X1);
                \path (H1) edge node [left=0.4 cm, above] {$w_{1i}$}(Xi);
                \path (H1) edge node [below=0.5 cm, right] {$w_{1d}$}(Xd);              
                \path (Hm) edge node [below =0.35 cm,right = 0.25cm] {}(X1);
                \path (Hm) edge node [above =0.25 cm,left = 0.15cm] {}(Xi);
                \path (Hm) edge node [above =0.25 cm,left = 0.15cm] {}(Xd);
        \end{tikzpicture}
  \caption{A RBM with $d$ visible and $m$ hidden units.}
  \label{fig:RBM}
 \end{figure}
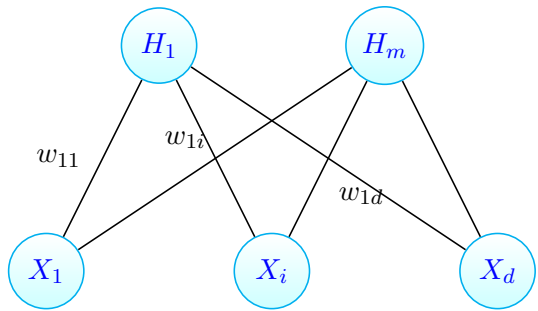
A RBM is parametrized by $\lambda = (W, a,b)$, where $W$ is the weight matrix of the connections between the visible and hidden units, and $a,b$ are the bias vectors of the visible and hidden layers, respectively. Each configuration $(X=x,H=h)$ of a RBM is associated with the following energy
\begin{equation}
E_\lambda(x,h) = -(a^Tx + b^Th + x^TWh)\label{eq:E}
\end{equation}
which defines the probability of the configuration
\begin{equation}
p_\lambda(X=x,H=h) = \frac{e^{-E_\lambda(x,h)}}{Z},\notag
\end{equation}
where $Z \equiv \sum_{x,h}e^{-E_\lambda(x,h)}$ is the \textit{partition function}. The bipartite structure of the RBM implies factorial conditional probabilities
\begin{equation}
p_\lambda(X|H) = \prod_i p_\lambda(X_i|H),\;\ \ \  p_\lambda(H|X) = \prod_j p_\lambda(H_j|X),\notag
\end{equation}
given by
\begin{align}
&p_\lambda(X_i=1|H) = \sigma (a_i+W_{i.}H) \notag\\
&p_\lambda(H_j=1|X) = \sigma (b_j + X^TW_{.j}),\notag
\end{align}
where $\sigma(z)$ is the sigmoid function defined in equation~\eqref{eq:sigmoid}, $W_{i.}$ is the $i$-th row of $W$ and $W_{.j}$ is its $j$-th column.

Given iid training data $x^{(1)},..,x^{(n)}~\sim p_\theta(X)$, the RBM parameters $\lambda = (W,a,b)$ are typically tuned to maximize the log-likelihood of the training data, where the likelihood that the RBM associates with a vector $x$ is given by
\begin{equation}
p_\lambda(X = x) =  \sum_h p_\lambda(X = x,H = h).\notag
\end{equation}
A popular approach to learn the RBM parameters is via gradient-based optimization, where the gradients are approximated using contrastive divergence \citep{hinton2006fast, bengio2009learning}.


\section {RBM in the Conditional Independence Case} \label{sec:condIndCase}
In this section we show that given observed data $x^{(1)},\ldots,x^{(n)} \in \mathbb \{0,1 \}^d$ from the conditional independence model of Eq.~\eqref{eq:condInd}, the posterior probabilities of the true, unknown labels $y^{(1)},\ldots,y^{(n)}$ can be consistently estimated via a RBM with a single hidden node. 

We begin by showing that there is a bijective map from the parameters $\lambda$ of a RBM with a single hidden node to the parameters $\theta$ of the conditional independence model, such that the joint distribution specified by the RBM is {\em equivalent} to that of the conditional independence model.

\begin{lemma}\label{lemma:map}
The joint probability $p_\lambda(X=x,H=y)$ of a RBM with parameters $\lambda = (a,b,W)$ is equivalent to the joint probability $p_\theta(X=x,Y=y)$ of a conditional independence model with parameters $\theta=(\{\psi_i\}, \{\eta_i\},\pi)$ given by
\begin{align}
& \psi_i \equiv \sigma(a_i+W_i),\; \eta_i \equiv 1-\sigma(a_i)\notag\\
&\pi \equiv \frac{\sum_{x \in \{0,1\}^d} e^{a^Tx + b  + x^TW}}{\sum_{x \in \{0,1\}^d}\left(e^{a^Tx} + e^{a^Tx + b + x^TW}\right)} \notag
\end{align} 
Furthermore, the map $\lambda \mapsto \theta$ is a bijection.
\end{lemma}

We are now ready to prove the main result of this section, namely, that the posterior distribution of the true labels $y^{(1)},\ldots, y^{(n)}$ can be consistently estimated by a RBM with a single hidden node. To do so, we rely on a special case of a result proved by~\citet{chang1996full}, that provides conditions under which the parameters of the conditional independence model are identifiable.

\begin {lemma} \label{lemma:condInd}
Let $x^{(1)},..., x^{(n)}$ be observed data from the conditional independence model, specified by $p_\theta$. Assume that $\theta$ is such that for each $i=1,\ldots,d$, $X_i$ is not independent of $Y$ (i.e., each classifier is not just a random guess), and that $d \ge 3$.
 Let $\hat{\lambda}_\text{MLE}$ be a maximum likelihood parameter estimate of a RBM with a single hidden node. 
Then the RBM posterior probability $p_{\hat{\lambda}_\text{MLE}}(H=1|X=x)$ converges to the true posterior 
$p_\theta(Y=1|X=x)$, as $n \rightarrow \infty$.
\end{lemma} \label{lemma:RBMSolves}

\begin{remark} \label{remark:flips}
The identifiability of the parameters is up to a single global $0/1$ label flip. This  means that one recovers either $p_\theta(Y=y|X)$ or $p_\theta(Y=1-y|X)$. Assuming that on average, the $X_i$'s are more accurate than a random guess, this sign ambiguity can be resolved by comparing the predictions to the majority vote decision.
\end{remark}

\begin{remark}
Lemma~\ref{lemma:condInd} assumes that we found the MLE of the RBM parameters. Obtaining such a MLE is problematic for two main reasons. First, RBMs are typically trained to maximize a proxy for the likelihood, as the true likelihood is not tractable. Second, the RBM likelihood function is not concave, hence there are no guarantees that after training a RBM one obtains the maximum likelihood parameter $\hat{\lambda}_\text{MLE}$.

\end{remark}


\section {RBM-based Deep Neural Net} \label{sec:RBM-DNN}
In many practical settings, the variables $X_1,\ldots,X_d$ are not conditionally independent. 
Fitting a conditionally independent model to such data may yield highly sub-optimal predictions for the true labels $y_i$. 
To tackle this general case, we propose to train a RBM-based Deep Neural Net (DNN) and use it to estimate the posterior probabilities $p_\theta(Y|X)$.
In such a DNN, the hidden layer of each RBM is the input for the successive RBM. As suggested by \citet{hinton2006fast}, the RBMs are trained one at a time, bottom to top, i.e., the DNN is trained in a layer-wise fashion.
Specifically, given training data $x^{(1)},\ldots,x^{(n)} \in \{0,1\}^d$, we start by training the bottom RBM, and then obtain the first layer hidden representation of the data by sampling $h^{(i)}$ from the conditional RBM distribution $p_\lambda (H|X = x^{(i)})$. The vectors $h^{(1)},\ldots,h^{(n)}$ are then used as a training set for the second RBM and so on. 

In the case considered in this manuscript, where the true label $y$ is binary, the upper-most RBM in the DNN has a single hidden unit, from which the posterior probability $p_\theta(Y|X)$ can be estimated. Such a DNN is depicted in Figure~\ref{fig:dnn}.
\tikzset{decorate sep/.style 2 args=
{decorate,decoration={shape backgrounds,shape=circle,shape size=#1,shape sep=#2}}}

\begin{figure}[ht!]
  \centering
  \begin {tikzpicture}[auto ,node distance =4 cm and 5cm ,on grid ,
                semithick ,
                state/.style ={ circle ,top color =white , bottom color = processblue!20 ,
                        draw,processblue , text=blue , minimum width =1.2 cm}]
                \node[state] (Yhat) at (10,6) {$\hat{Y}$};
          \node[state] (H21) at (7,4) {$H^2_1$};
                \node[state] (H2i) at (10,4) {$H^2_i$};
                \node[state] (H2m2) at (13,4) {$H^2_{m_2}$};
                \node[state] (H11) at (7,2) {$H^1_1$};
                \node[state] (H1i) at (10,2) {$H^1_i$};
                \node[state] (H1m1) at (13,2) {$H^1_{m_1}$};
                \node[state] (X1) at (7,0) {$X_1$};
                \node[state] (Xi) at (10,0) {$X_i$};
                \node[state] (Xd) at (13,0) {$X_d$};
                \path (H11) edge node [above =0.15 cm,left = 0.15cm] {}(X1);
                \path (H11) edge node [above =0.15 cm,left = 0.15cm] {}(Xi);
                \path (H11) edge node [above =0.15 cm,left = 0.15cm] {}(Xd);
                \path (H1i) edge node [above =0.15 cm,left = 0.15cm] {}(X1);
                \path (H1i) edge node [above =0.15 cm,left = 0.15cm] {}(Xi);
                \path (H1i) edge node [above =0.15 cm,left = 0.15cm] {}(Xd);            
                \path (H1m1) edge node [above =0.15 cm,left = 0.15cm] {}(X1);
                \path (H1m1) edge node [above =0.15 cm,left = 0.15cm] {}(Xi);
                \path (H1m1) edge node [above =0.15 cm,left = 0.15cm] {}(Xd);
                \path (H21) edge node [above =0.15 cm,left = 0.15cm] {}(H11);
                \path (H21) edge node [above =0.15 cm,left = 0.15cm] {}(H1i);
                \path (H21) edge node [above =0.15 cm,left = 0.15cm] {}(H1m1);
                \path (H2i) edge node [above =0.15 cm,left = 0.15cm] {}(H11);
                \path (H2i) edge node [above =0.15 cm,left = 0.15cm] {}(H1i);
                \path (H2i) edge node [above =0.15 cm,left = 0.15cm] {}(H1m1);          
                \path (H2m2) edge node [above =0.15 cm,left = 0.15cm] {}(H11);
                \path (H2m2) edge node [above =0.15 cm,left = 0.15cm] {}(H1i);
                \path (H2m2) edge node [above =0.15 cm,left = 0.15cm] {}(H1m1); 
                \path (Yhat) edge node [above =0.15 cm,left = 0.15cm] {}(H21);
                \path (Yhat) edge node [above =0.15 cm,left = 0.15cm] {}(H2i);
                \path (Yhat) edge node [above =0.15 cm,left = 0.15cm] {}(H2m2); 
                \draw[decorate sep={0.5mm}{4mm},fill] (8.1,0) -- (9,0);
                \draw[decorate sep={0.5mm}{4mm},fill] (11.1,0) -- (12,0);
                \draw[decorate sep={0.5mm}{4mm},fill] (8.1,2) -- (9,2);
                \draw[decorate sep={0.5mm}{4mm},fill] (11.1,2) -- (12,2);
                \draw[decorate sep={0.5mm}{4mm},fill] (8.1,4) -- (9,4);
                \draw[decorate sep={0.5mm}{4mm},fill] (11.1,4) -- (12,4);
                
        \end{tikzpicture}
  \caption{A sketch of RBM-based DNN with two hidden layers.}
  \label{fig:dnn}
 \end{figure}
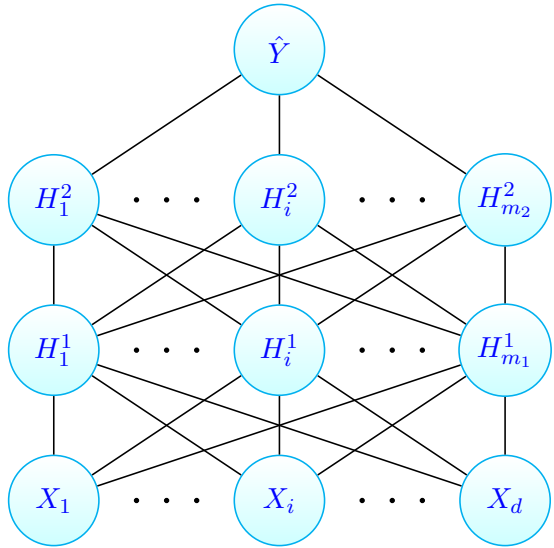


\subsection {Motivation}\label{sec:mot}
Deep learning algorithms have recently achieved state-of-the-art performance in a wide range of applications~\cite{lecun2015deep}. 
While a rigorous theoretical understanding of deep nets is still lacking, many researchers believe that a key property in their success is their ability to disentangle factors of variation in the inputs; see for example~\citet{bengio2013representation},~\citet{tishby2015deep}, and~\citet{mehta2014exact}. 
That is, as one moves through the net, the hidden units become  less statistically dependent. 
We have seen in Section~\ref{sec:condIndCase} that given a representation in which the units are independent conditional on the true label, a single node RBM gives a consistent estimation of the true label posterior probability. 
Propagating the data through several RBM layers can hence be seen as a processing of the data, which reduces the conditional dependence of the units while preserving most of the information on the true label $Y$.
In Section~\ref{sec:experimentalResults} we will demonstrate cases where such decoupling does indeed happen in practice, i.e., although the original input variables $X_i$'s are not conditionally independent given the true label $Y$, after training, the units in the uppermost hidden layer are, remarkably,  approximately conditionally independent. Thus, the assumptions of the conditional independence model apply (with respect to the uppermost hidden layer $H^\text{last}$), and therefore one is able to consistently estimate the label posterior probability, $\Pr(Y|H^\text{last})$, as in Section~\ref{sec:condIndCase}. 

Another motivation for using deep nets with several hidden layers for unsupervised ensemble learning is their rich expressive power. 
In our setting, we wish to approximate the posterior probability $p(Y|X)$, which in general may be a complicated nonlinear function of $X$. When $p(Y|X)$ cannot be accurately estimated by a RBM with a single hidden node (i.e., when the conditional independence assumption of Dawid and Skene does not hold), a better approximation may be obtained from a deeper network.
Several works show that there exist functions that are significantly more efficiently represented by deeper networks, compared to shallower ones, where efficiency corresponds to the number of units. For example, \citet{montufar2014number} show that deep networks with piece-wise linear activations can represent functions with greater number of linear regions compared to shallow networks with the same number of units. In a recent work,~\citet{eldan2015power} give an example for a radial function that can be efficiently computed by a 3-layer network, while requiring exponentially many units to be approximated accurately by a 2-layer network.

Finally, we would like to emphasize that a RBM-based DNN is a discriminative model to estimate the posterior $p(Y|X)$.
In general, it may not correspond to any generative model~\cite{arora2015deep}. Indeed, there is no guarantee that the marginal distributions implied by two adjacent RBMs match. Yet, it can be shown (see Appendix~\ref{app:varInf}) that stacking RBMs is a variational inference procedure assuming a specific class of data generation models.
The nature of approximation of a top down generative model, where the data $X$ is generated from a label $Y$, by a RBM-based DNN is explored in Appendix~\ref{app:approxMeanField}.


\subsection {Predicting the Label from a Trained DNN}
Given a trained DNN and a sample $x\sim p_\theta(X)$, the label $y$ is estimated by propagating $x$ through the network. Specifically, the units of each layer can be set by either (i) sampling from the conditional distribution given the layer below, i.e., $h_j \sim p_\lambda(h_j|x)$, or (ii) by MAP estimate, setting each hidden unit $h_j = \arg\max_{h_j \in \{0,1\}}p_\lambda(h_j|x)$. Since the first option is stochastic, one may propagate $x$ through the net multiple times and average the outputs $p(y|x)$ to obtain an approximation of $\mathbb{E}(Y|X=x)$. 
Experimentally, we found both options to be equally effective, while each option slightly outperforms the other in some cases.


\subsection {Choosing the DNN Architecture}\label{sec:svd}
The specific DNN architecture (i.e., number and sizes of layers) might have a dramatic effect on the quality of predictions. To determine the number of units in each layer we employed the following procedure: we first train a RBM with $d$ hidden units. Next, we compute the singular value decomposition of the weight matrix $W$, and determine its rank (i.e., the number of sufficiently large singular values). Given that the rank is some $m \le d$, we re-train the RBM, setting the number of hidden units to be $m$. If $m >1$, we add another layer on top of the current layer, and proceed recursively. The process stops when $m=1$, so that the last layer of the DNN contains a single node. We refer to this method as {\em the SVD approach}.
In our experiments, as a rule of thumb, we set $m$ to be the minimal number of singular values (in descending order) whose cumulative sum is at least 95\% of the total sum.

This method takes advantage of the co-adaptation of hidden units, which is a well known phenomenon in RBM training (see, for example,~\cite{hinton2012improving}). The term {\em co-adaptation} describes a situation where several hidden units tend to behave very similarly; this implies that the rank of the weight matrix might be small, although the number of hidden units may be larger.


\section{Experimental Results} \label{sec:experimentalResults}
In this section we compare the performance of the proposed DNN approach to several other approaches, and report experimental results obtained on four simulated data sets and eight real world data sets, from two different domains.
All our datasets, as well as the scripts reproducing the reported results are publicly available  at \url{https://github.com/ushaham/RBMpaper}.
\footnote{
Our scripts are based on the publicly available code in Hinton's website \url{http://www.cs.toronto.edu/~hinton/MatlabForSciencePaper.html}. }.

Specifically, we compare between the following unsupervised ensemble methods:
\begin{itemize}
\item \textbf{Vote.} Majority voting, which is the maximum likelihood prediction, assuming that all classifiers are conditionally independent and have the same accuracy.
\item \textbf{DS.}  Approximate maximum likelihood predictions under the Dawid and Skene model. Specifically, we use Spectral Meta Learner~\citep{parisi2014ranking}, and Restricted Likelihood~\citep{jaffe2014estimating}.
\item \textbf{CUBAM} The method of~\citet{welinder2010multidimensional}, which assumes conditional independence, but allows the accuracy of each classifier to vary across different regions of the input domain.
\item \textbf{L-SML} Latent SML \citep{jaffe2015unsupervised}. This method relaxes the conditional independence assumption to a depth 2 tree model.
\item \textbf{DNN} The approach presented in this manuscript, with the depth and number of hidden units in each layer determined by the SVD approach, described in Section~\ref{sec:svd}. 
\end{itemize}

Following ~\citet{jaffe2015unsupervised}, the performance measure we chose is the balanced accuracy, given by
\begin{align}
&\frac{\sum \mathbb{I}\{\text{true label is 0 and predicted label is 0} \}}{2\sum \mathbb{I}\{\text{true label is 0} \}} \notag \\
+ &\frac{\sum \mathbb{I}\{\text{true label is 1 and predicted label is 1} \}}{2\sum \mathbb{I}\{\text{true label is 1} \}}, \notag
\end{align}
where $\mathbb{I}\{\cdot\}$ is the indicator function.


\subsection {Simulated Datasets}\label{sec:simu}

In this experiment we carefully generated four synthetic datasets, in order to demonstrate the performance of the DNN approach in several specific scenarios. In all four datasets the observed data is a $ n \times d$ binary matrix, with input dimension $d=15$ and sample size $n=10,000$. A detailed description of the datasets generation process is given in Appendix~\ref{app:datasets}.
\begin{itemize}

\item \textbf{CondInd} A dataset where the conditional independence holds, and $10$ of the $15$ classifiers are in fact random guess.

\item \textbf{Tree15-3-1} A dataset generated from a depth-2 tree with layer sizes 1,3,15. Every node in the intermediate layer is connected to five nodes in the bottom layer. This dataset is generated from the model considered by L-SML, and does not satisfy the conditional independence assumption, as is shown in Figure~\ref{fig:tree}.

\item \textbf{LayeredGraph15-5-5-1} A dataset generated from a depth-3 layered graph, with layer sizes 1,5,5,15. In this case, the conditional independence assumption does not hold, although in practice the amount of dependence in the data is not high (see Figure~\ref{fig:datasets}).

\item \textbf{TruncatedGaussian.} Here \(X=(1+\mbox{sign}(Z))/2\), where the r.v. $Z$ follows a a mixture of two $d$-dimensional Gaussians with different means and same covariance matrix. The label $Y$ indicates the specific Gaussian from which $X$ is sampled. In this case, the data is highly dependent, as can be seen in Figure~\ref{fig:datasets}.

\end{itemize}
The results are summarized in Table~\ref{tab:results1}. Along with the five unsupervised methods, the table also shows the accuracy of a supervised learner and the estimated accuracy of the Bayes-optimal classifier. The supervised learner is a  Multi Layer Perceptron (MLP) with two hidden layers of sizes 4 and 2, that was trained on a dataset with $n=10,000$ samples (independent of the test dataset). The Bayes-optimal approximated accuracy was computed on a  sample of size $10,000$, with the true posterior probabilities of all \(2^{d}\) possible binary vectors estimated using a sample of size $10^6$ from the corresponding model.  
\begin{table*}[t]
\centering
\caption{Balanced accuracy of various unsupervised ensemble methods on the four synthetic datasets, along with a supervised learner (SUP), and the Bayes optimal classifier (Bayes-Opt). The results are presented as mean $\pm$ standard deviation, based on 5 repetitions, where in each repetition a new dataset was sampled from the model. The numbers in brackets denote the architecture of the DNN, found by the SVD approach.}
\vskip 0.15in
\begin{tabular}{ l || c | c  | c | c}
  \hline                        
  method  & condInd & Tree15-3-1 & LG15-5-5-1 & TG \\ \hline\hline      
  Vote  & 75.93 $\pm$ 0.5& 93.45 $\pm$ 0.19 & 76.61 $\pm$ 0.09 & 80.14 $\pm$ 0.4\\ \hline    
  DS  & \textbf{94.78 $\pm$ 0.13}& 92.68 $\pm$ 0.14 & 86.36 $\pm$ 0.2 & 82.03 $\pm$ 0.27\\ \hline
  CUBAM  & 91.96 $\pm$ 0.18 & 90.74 $\pm$ 0.3 & 77.12 $\pm$ 0.26 & 83.43 $\pm$ 0.31\\ \hline
  L-SML  & 55.94 $\pm$ 21.88 &\textbf{95.83 $\pm$ 0.15}& 85.87 $\pm$ 0.21 & 79.5 $\pm$ 1.35\\ \hline
  DNN  & \textbf{94.78 $\pm$ 0.13} (15-1)& 95.13 $\pm$ 0.71 (15-3-1) & \textbf{86.83 $\pm$ 0.2} (15-4-1) & \textbf{88.09 $\pm$ 0.52} (15-3-1)\\ \hline\hline 
  SUP  & 94.45 $\pm$ 0.11 & 95.54 $\pm$ 0.27 & 87.01 $\pm$ 0.18 & 90.8 $\pm$ 0.4 \\ \hline
  Bayes-Opt  & 95.32 & 96.12 & 87.05 & 91.39\\ \hline
  \end{tabular} 
  \label{tab:results1}
\end{table*}

On all of the above datasets, the DNN always outperformed the majority vote rule and CUBAM.
On the CondInd dataset, the DNN performs similarly to DS, and significantly better than the other methods. Despite being unsupervised, on this dataset both methods perform slightly better than the specific supervised learner we considered, and around the Bayes-optimal accuracy. The architecture determined by the SVD approach in this case is indeed a single RBM (with a single hidden node). The weight matrix of the RBM is shown in Figure~\ref{fig:rbmCondind}, and corresponds to the fact that only the first five classifiers actually contain information about the true label in this dataset.
\begin{figure}[]
  \centering
  \includegraphics[width=3in]{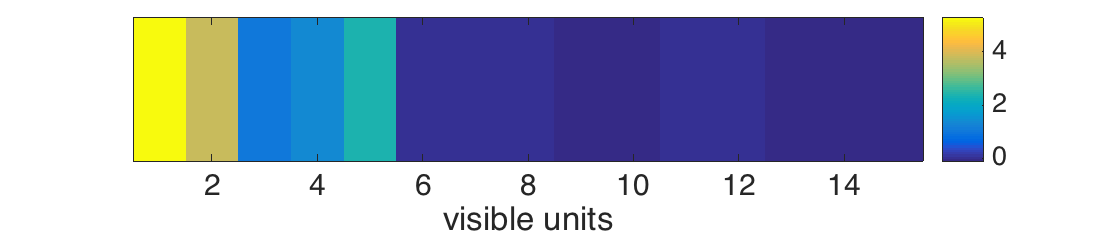}
  \caption{The RBM weight vector on the condInd dataset. The hidden unit is strongly connected only to the first five visible units, reflecting the fact that in an unsupervised manner, the RBM detected that the remaining units are random guess classifiers.}
  \label{fig:rbmCondind}
 \end{figure}

Figure~\ref{fig:paramsRecovery} shows the recovery of the true conditional independence model parameters $\{\psi_i, \eta_i \}$ of a similar conditional independent dataset (however with no random guess classifiers) from a RBM with a single hidden node, using the map in Lemma~\ref{lemma:map}.
\begin{figure}[]
  \centering
  \includegraphics[width=3in,height=1.8in]{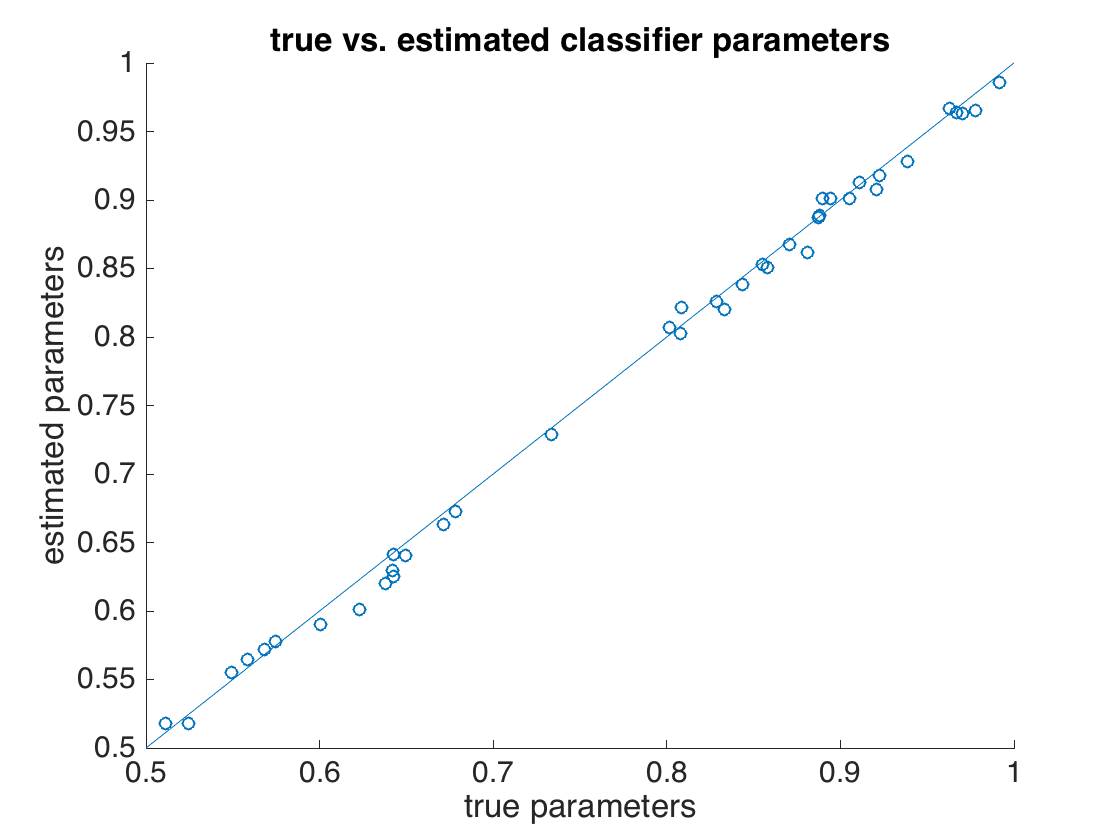}
  \caption{Recovery of the conditional Independence model parameters $\{\psi_i, \eta_i \}$ from a RBM with a single hidden node, on a dataset sampled from a conditional independence model.  The parameters were uniformly sampled from $[0.5,1]$. Each circle corresponds to a single parameter (e.g., $\psi_i$ for some $i$). For convenience, the identity line was added to the plot.}
  \label{fig:paramsRecovery}
 \end{figure}
 
On the Tree15-3-1 dataset, L-SML, which is tailored for data generated by a tree, outperforms the DNN. This result is expected, since it can be shown that the distribution of the bottom two layers of a tree cannot be parametrized as a RBM (see Appendix~\ref{app:approxMeanField}). Still, the DNN performs significantly better than DS, CUBAM and majority vote, and not far from the supervised learner and the optimal Bayes classifier. Figure~\ref{fig:tree} shows the correlation matrix at the input and hidden layers, as well as the first layer weight matrix, demonstrating that the DNN captured the true data generation model. Consequently, the 3 hidden units are nearly conditionally uncorrelated given the label $y$. 

\begin{figure*}[h!]
  \centering
  \includegraphics[width=2.5in,height=2in]{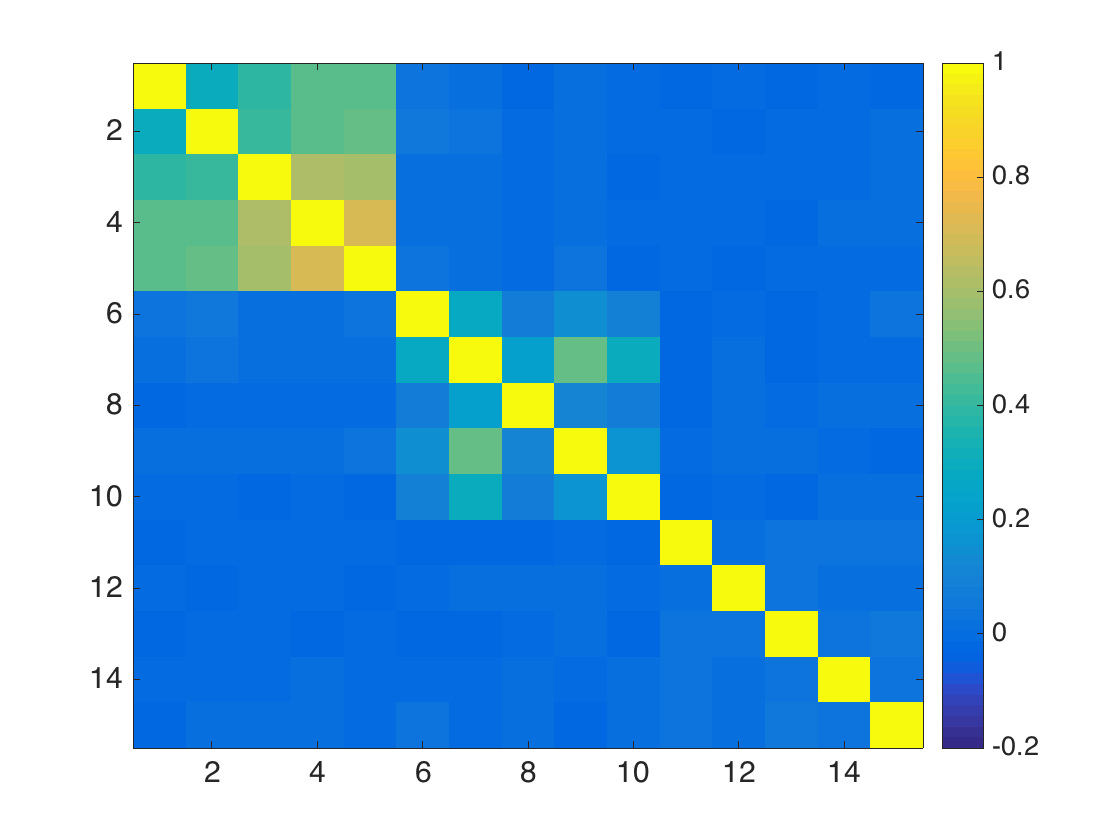}
  \includegraphics[width=2.5in,height=2in]{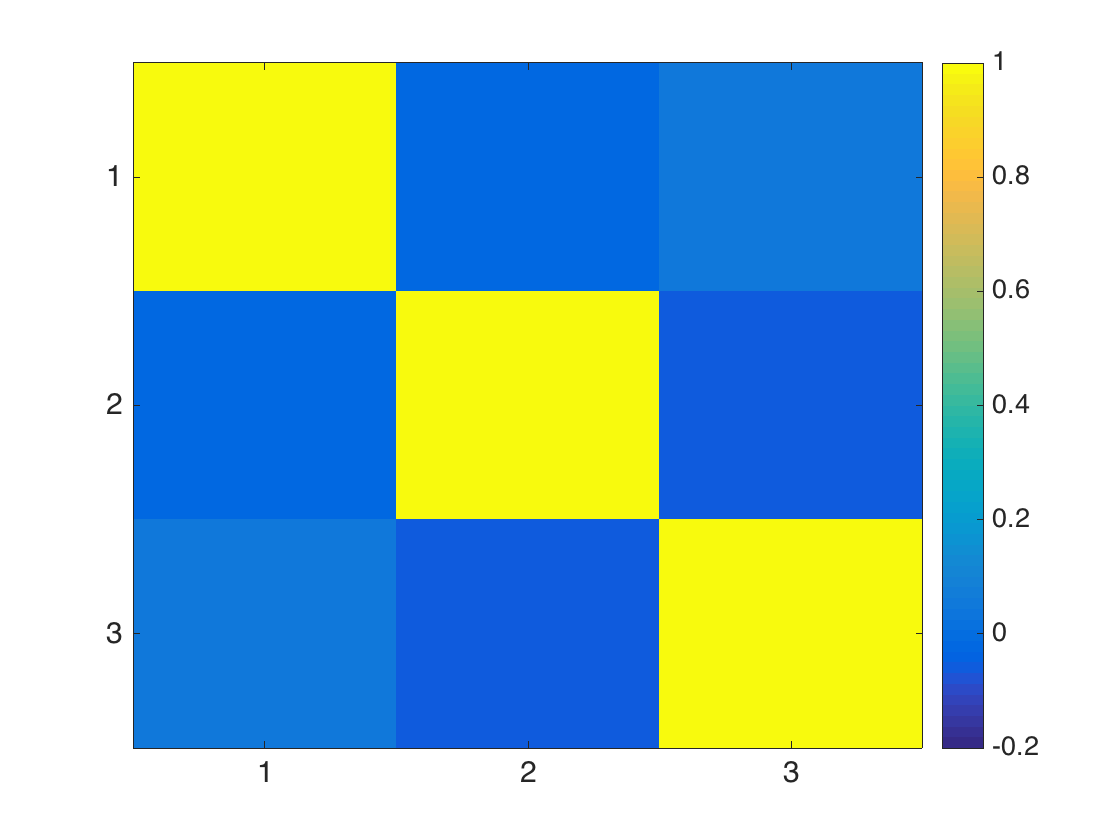}
  \includegraphics[width=2.5in,height=2in]{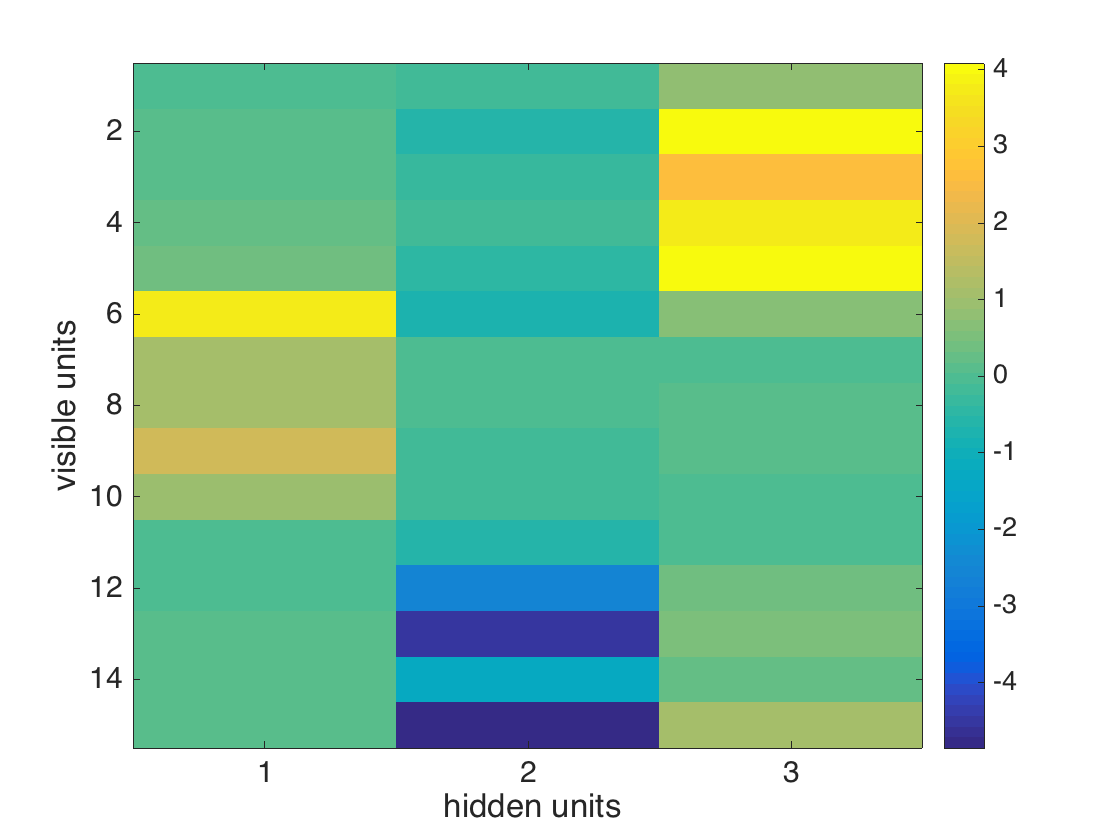}
  \caption{The Tree15-3-1 experiment. Top left: correlation matrix of the input data for the $y=0$ class. The first and middle five $X_i's$ are not conditionally independent of each other.  Top right: correlation matrix of the hidden layer of the DNN for the $y=0$ class. The hidden units are approximately uncorrelated. Bottom: weight matrix of the bottom RBM of the DNN, showing that each hidden unit is strongly connected to 5 visible units, as in the original data generation model.}
  \label{fig:tree}
 \end{figure*}
 
Figure~\ref{fig:cumulative} shows the cumulative proportion of the singular values on the condInd and Tree15-3-1 datasets, which explains the architecture determined by the SVD approach for both datasets.
 \begin{figure}[t]
  \centering
  \includegraphics[width=3in,height=1.5in]{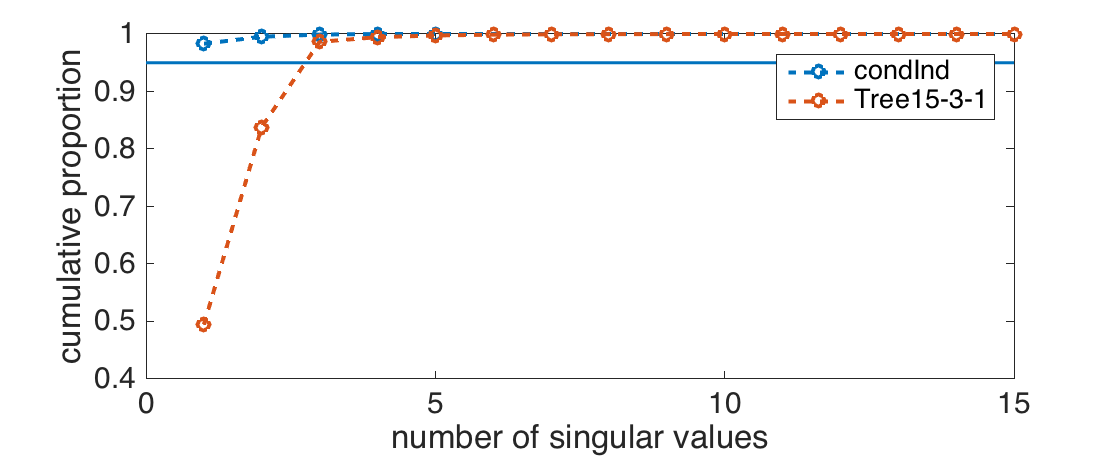}
  \caption{Cumulative proportion of singular values on the condInd and Tree15-3-1 datasets. While in the condInd case the first singular value is more than 95\% of the total sum of singular values, the first three singular values are needed on the Tree15-3-1 dataset. The horizontal line at 0.95 is added to the plot for convenience.}
  \label{fig:cumulative}
 \end{figure}

On the LayeredGraph15-5-5-1 dataset, while outperforming the other methods, the DNN achieved accuracy close to the supervised learner and the Bayes optimal accuracy; however, the chosen DNN architecture is different from the one of the true data generation model. 

The conditional independence assumption is strongly violated in the case of the TruncatedGaussian dataset. Here the DNN performs better than all other methods by a large margin.


\subsection{Real-World Datasets}
In this section we experiment with two groups of datasets, from two different domains, as follows:
\begin{itemize}

\item \textbf{DREAM} 
Three datasets from the DREAM mutation calling challenge~\mbox{\cite{ewing2015combining}}; this challenge is an international effort to improve standard methods for identifying cancer-associated mutations and rearrangements in whole-genome sequencing data. 
The accuracy of current variant calling algorithms is not optimal due to sequencing errors, other experimental factors, parametric choices in each algorithm and preprocessing and filtering decisions. 
Unsupervised ensemble learning of multiple variant callers is expected to provide more robust predictions. 
One of the goals of this challenge is to
develop a state-of-the-art meta pipeline for somatic mutation
detection, to output accurate as possible mutation calls associated with
cancer.
Specifically, we used three datasets, (S1, S2, S3) containing the predictions of classifiers that determine the presence or absence of of mutations in genome sequencing data. The data is available at~\citep{DREAM}. In S1, $d = 124$, $n = 92,362$. In S2, $d$ = 114, $n$ = 70,561. In S3, $d = 99$, $n = 78,643$.

\item \textbf{Magic} Forty datasets, which are constructed from the Magic dataset in the UCI repository, available at \url{https://archive.ics.uci.edu/ml/datasets/MAGIC+Gamma+Telescope}. 
This dataset contains $n = 19,020$ instances with 11 attributes, which consists of physical measurements of gamma particles; the learning task is to classify each instance as background or high energy gamma rays.
Each of the five datasets we constructed contains binary predictions of $d=16$ classifiers, obtained in the Weka machine learning software. 
The 16 classifiers belong to four groups: four random forest classifiers, three logistic trees classifiers, four SVM classifiers, and five naive Bayes classifiers. This setting is adopted from~\cite{jaffe2015unsupervised}. The group of SVM classifiers is highly correlated, as well as the group of Naive Bayes classifiers, as can be seen in Appendix~\ref{app:magic}.
Each of the forty datasets was obtained by predictions of the same classifiers, however trained on a different subset of the original Magic dataset (a random subset of size 500 each time).
\end{itemize}

Table~\ref{tab:results3} shows the performance of the various methods on the DREAM datasets.
\begin{table*}[t]
\centering
\begin{tabular}{ l || c | c | c | c | c}
  \hline                        
  Dataset  & Vote  & DS & CUBAM & L-SML  & DNN \\ \hline\hline  
  S1  &  97.2 * & 98.3 * & 92.31 & \textbf{98.4} * & \textbf{98.42 $\pm$  0.0} (124-1)\\ \hline  
  S2  &  96   * & 97.2 * & 69.19 & \textbf{97.7} * & 97.55 $\pm$  0.01 (114-1)\\ \hline
  S3  &  95.7 * & 97.7 * & 87.65 & 98.2 * & \textbf{98.51 $\pm$  0.01} (99-25-1) \\ \hline

  \end{tabular} 
  \caption{Balanced accuracy of various methods on the DREAM datasets S1, S2 and S3. DNN results are averaged over 5 repetitions, and are presented as mean $\pm$ standard deviation. The numbers in brackets denotes the architecture of the DNN, found by the SVD approach. * results reported in~\citep{jaffe2015unsupervised} }
  \label{tab:results3}
\end{table*}
As can be seen, the DNN and L-SML performs similarly on S1, while the former performs better on S3 and the latter on S2. The two methods outperform the majority vote rule, DS and CUBAM on all three datasets.
Remarkably, the hidden representation on the S3 dataset is such that the units are perfectly uncorrelated, conditioned on the hidden label. This is shown in Figure~\ref{fig:S3}.
\begin{figure*}[h!]
  \centering
  \includegraphics[width=3.0in,height=2.0in]{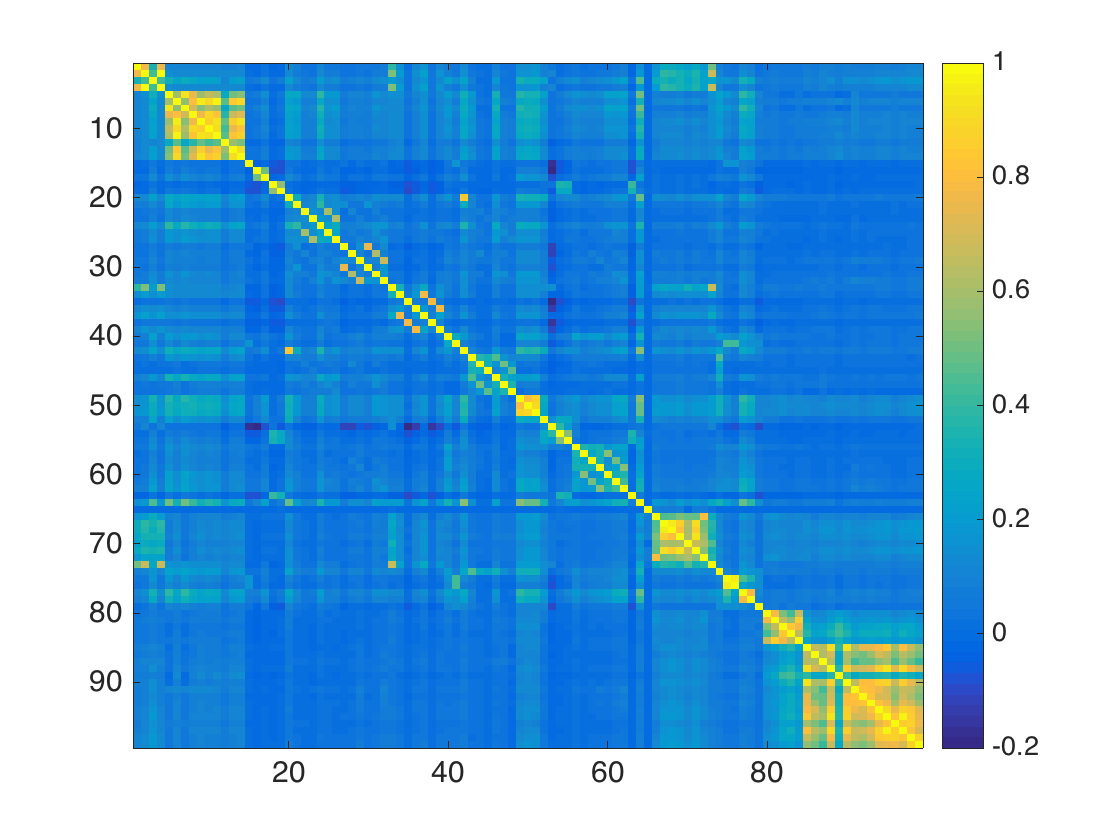}
  \includegraphics[width=3.0in,height=2.0in]{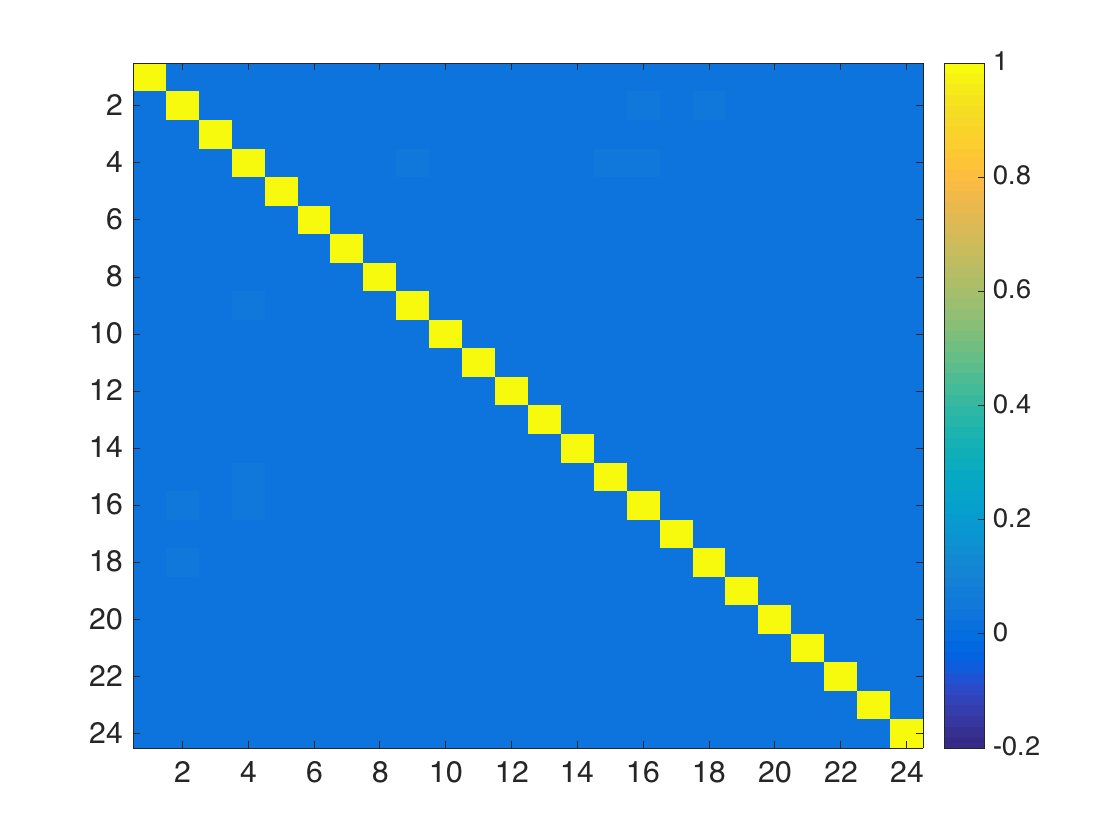}
  \caption{correlation matrices of the input (left) and hidden (right) layers of the DNN on the S3 dataset, for the $y=0$ class. Remarkably, the hidden units are almost perfectly uncorrelated, conditioned on the class.}
  \label{fig:S3}
 \end{figure*}

The results on the Magic datasets are shown in Figure~\ref{fig:magicResults}. On most of these datasets, the DNN outperforms all other methods, with a relatively large margin. On all forty datasets, the SVD approach yielded a 15-3-1 architecture. 
\begin{figure}[h!]
  \centering
  \includegraphics[width=3.0in,height=2.0in]{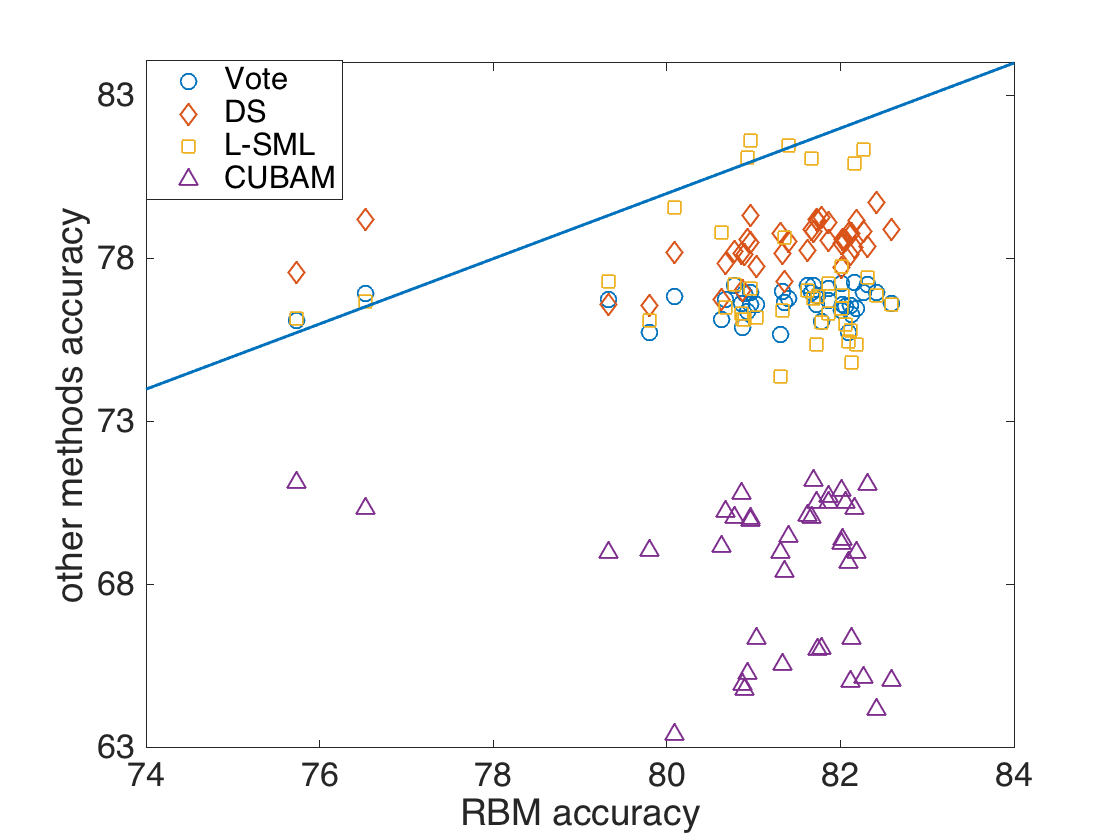}
  \caption{Performance of the various methods on the Magic datasets. For convenience, the identity line is added to the plot. Most of the points are below the identity line, which indicates that the DNN tend to outperform all other methods on these datasets.}
  \label{fig:magicResults}
 \end{figure}
 
To summarize our experiments, we observed that RBM-based DNN performs at least as well and often better than various other methods, on both simulated and real datasets, and that the SVD approach can serve as an effective tool for determination of the DNN architecture. 

We remark that in our experiments, we observed that RBMs tend to be highly sensitive to hyper-parameter tuning (such as learning rate, momentum, regularization type and penalty), and these hyper-parameters need to be carefully tuned.
To obtain a reasonable hyper-parameter setting we found it useful to apply the random configuration sampling procedure, proposed in~\citep{bergstra2012random}, and evaluate different models by average log-likelihood approximation, (see, for example, \citep{salakhutdinov2008quantitative} and the corresponding MATLAB scripts in~\citep {Rus}).


\section {Summary and Discussion}\label{sec:conclusions}
We demonstrated how deep learning techniques can be used for unsupervised ensemble learning,
and showed that the DNN approach proposed in this manuscript often performs at least as well and often better than state-of the art methods, especially when the conditional independence assumption made by \citet{dawid1979maximum} does not hold. 

Possible directions for future research include extending the approach to multiclass problems, possible using Discrete RBMs~\cite{montufar2013discrete}, theoretical analysis of the SVD approach, and information theoretic analysis of the de-correlation, while preserving label information, that occurs while propagating data through a RBM-based DNN.

\section*{Acknowledgements}
The authors would like to thank George Linderman, Alex Cloninger, Tingting Jiang, Raphy Coifman, Sahand Negahban, Andrew Barron, Alex Kovner, Shahar Kovalsky, Maria Angelica Cueto, Jason Morton, and Brend Strumfels for their help.

\bibliography{RBM2}
\bibliographystyle{apalike}


\clearpage

\appendix

\section{Proof of Lemma~\ref{lemma:map}} \label{app:lemmaMap}

\begin{proof}
We will define $\theta$ so that for every $x,y$, $p_\theta(X_i=x_i|Y = y) = p_\lambda(X_i = x_i|H=y)$ and $p_\theta(Y=y) = p_\lambda(H=y)$.

Since the weight matrix $W$ has dimension $d \times 1$ in this case, it is a vector, which we will denote as $w$. Recall that 
\begin{equation}
p_\lambda(X_i=1|H=y) = \sigma(a_i + w_iy), \notag
\end{equation}
hence we define
\begin{equation}
\psi_i \equiv \sigma(a_i + w_i) \notag
\end{equation}
and
\begin{equation}
\eta_i \equiv 1-\sigma(a_i). \notag
\end{equation}
Finally, recall that
\begin{align}
p_\lambda(H=1) &= \frac{\sum_{x \in \{0,1\}^d} e^{-E_\lambda(x,1)}}{\sum_{x \in \{0,1\}^d,\; h \in \{0,1 \}}e^{-E_\lambda(x,h)}} \notag\\
&= \frac{\sum_{x \in \{0,1\}^d} e^{a^Tx + b  + x^Tw}}{\sum_{x \in \{0,1\}^d,}e^{a^Tx} + e^{a^Tx + b + x^Tw}}, \notag 
\end{align}
where $E_\lambda$ is the energy function given in equation~\eqref{eq:E},
hence we set 
\begin{equation}
\pi \equiv \frac{\sum_{x \in \{0,1\}^d} e^{a^Tx + b  + x^Tw}}{\sum_{x \in \{0,1\}^d,}\left(e^{a^Tx} + e^{a^Tx + b + x^Tw}\right)}. \label{eq:pi}
\end{equation}
To see that the map $\lambda \mapsto \theta$ is 1:1, note that $a_i$ uniquely determines $\eta_i$, hence $(a_i,w_i)$ uniquely determine $(\psi_i, \eta_i)$. Lastly, rearranging equation~\eqref{eq:pi} we get
\begin{align}
&\pi\sum_{x \in \{0,1\}^d} \left( e^{a^Tx} + e^{a^Tx + b + w^Tx}\right) = \sum_{x \in \{0,1\}^d} e^{a^Tx + b + w^Tx}\notag \\
\Rightarrow &\pi\sum_{x \in \{0,1\}^d}e^{a^Tx} = (1-\pi)e^b \sum_{x \in \{0,1\}^d}e^{a^Tx + w^Tx}\notag \\
\Rightarrow & e^b = \frac{\pi}{1-\pi} \frac{\sum_{x \in \{0,1\}^d}e^{a^Tx}}{\sum_{x \in \{0,1\}^d}e^{a^Tx + w^Tx}},\notag
\end{align}
so that given $(a,W)$, $\pi$ is uniquely determined by $b$.
Showing that the map $\lambda \mapsto \theta$ is a also subjective is straightforward. Hence it is a bijection.
\end{proof}


\section{Proof of Lemma~\ref{lemma:condInd}} \label{app:lemmaRBMSolves}

\begin{proof}
Since $d \ge 3$ and for each $i$, $X_i$ is not independent of $Y$, by~\citet{chang1996full}, the parameter $\theta$ of the conditional independence model is identifiable. 
Since the map $\lambda \mapsto \theta$ in Lemma~\ref{lemma:map} is a bijection, there exists $\lambda$ corresponding to $\theta$, which is therefore identifiable as well.
By the consistency property of the MLE (see, for example,~\citep{casella2002statistical}), 
\begin{equation}
\lim_{n \rightarrow \infty} \hat{\lambda}_\text{MLE} = \lambda. \notag
\end{equation}
Since $p_\lambda(H=1|X)$ is continuous in $\theta$, one obtains
\begin{equation}
p_{\hat{\lambda}_\text{MLE}}(H=1|X) \rightarrow p_\lambda(H=1|X). \notag
\end{equation}

Finally, note that Lemma~\ref{lemma:map} implies, in particular, that under the map $\lambda \mapsto \theta$
\begin{equation}
p_\lambda(H=1|X) = p_\theta(Y=1|X), \notag
\end{equation}
which completes the proof.
\end{proof}


\section{Stacking RBMs as a Variational Inference Procedure} \label{app:varInf}

Variational inference is a common approach to tackle complicated probability estimation problems (see, for example,~\cite{bishop2006pattern, fox2012tutorial}, and a recent review~\cite{blei2016gvariational}). Specifically, let $p$ be a target probability distribution that we want to approximate. In variational inference we define a family of approximate distributions $\mathcal{D} = \{q_\alpha: \alpha \in \mathcal{A} \}$, and then perform optimization to find the member of $\mathcal{D}$ that is closest to $p$ in Kullback-Leibler distance. A key idea is that the family $\mathcal{D}$ is flexible enough to contain a distribution close to $p$, yet simple enough to perform optimization over. For example, a popular choice is to take $\mathcal{D}$ as the collection of factorized distributions, i.e., of the form $q_\alpha(X) = \prod_i q_\alpha(X_i)$.
In this section, we motivate the use of RBM-based DNN by considering a specific data generation model, and showing that training a stack of RBMs on data generated by this model is in fact a variational inference procedure. 

The generative model we consider is a two layer Deep Belief Network (DBN), which played an important role in the emergence of deep learning in 2006~\cite{hinton2006fast}. The DBN we consider generates data $Y \in \{0,1\}$, $H \in \{0,1\}^m$, $X \in \{0,1\}^d$ via the probability distribution
\begin{equation}
p_\theta(X,H,Y) \equiv p_{\theta_1}(X,H)p_{\theta_2}(Y|H) \notag
\end{equation}
where $X,H$ form a RBM (parametrized by $\theta_1$).

We observe data $x^{(1)}\ldots x^{(n)}$ from $p_\theta(X)$ and our goal is to estimate the posterior $p_\theta(y^{(i)}|x^{(i)})$ for $i=1,\ldots n$. The posterior can be written as
\begin{equation}
p_\theta(Y|X) = \mathbb{E}_{h \sim p_{\theta_1}(H|X)}P_{\theta_2}(Y|H=h). \notag
\end{equation}

\citet{cueto2010geometry} showed that as long as $m$ is not too large comparing to $d$, RBMs are locally identifiable, i.e., identifiable up to order and flips of hidden units (Jason Morton, personal communication).
Therefore, when training a RBM with $m$ hidden units on $x^{(1)}\ldots x^{(n)}$, by the consistency property of the MLE ~\cite{casella2002statistical} the MLE $\hat{\theta}_{1\text{MLE}}$ will converge to the true parameter $\theta_1$ as $n \rightarrow \infty$. Hence, when $n$ is large enough, the vectors $h^{(i)}$ obtained from the (trained) RBM are in fact samples from $p_{\theta_1}(H|X = x^{(i)})$.

At the next step, the vectors $h^{(1)}\ldots h^{(n)}$ are used to train a second RBM, with a single hidden node. Observe that in the data generation model considered in this section, $p_\theta(H|Y)$ does not factorize. 
The factorized distribution $p_\lambda(H|Y)$ that minimizes $\KL(p_{\theta_2}(H|Y)\| p_\lambda(H|Y))$ is given by
\begin{equation}
p_\lambda(H_i|Y) = p_{\theta_2}(H_i|Y)\notag
\end{equation}
~\citet{bishop2006pattern} (Chapter 10).
By Lemma~\ref{lemma:map}, we know that the distribution 
\begin{equation}
p_\lambda(H,Y) = p_\theta(Y)\prod_ip_{\theta_2}(H_i|Y)\label{eq:77}
\end{equation}
is equivalent to a RBM. 
Finally, by Lemma~\ref{lemma:condInd}, the distribution~\eqref{eq:77} is consistently estimated by a RBM trained on vectors $h^{(1)}\ldots h^{(n)}$, and is thus a variational inference procedure.


\section{Stacking RBMs as an Approximation for a Directed Top-Down Model} \label{app:approxMeanField}

Assume that the data is generated by a Markov chain $Y \rightarrow H \rightarrow X$, where $Y \in \{0,1\}$, $H \in \{0,1\}^m$, $X \in \{0,1\}^d$. We further assume that the distributions $p_\theta(X|H),\; p_\theta(H|Y)$ factorize, i.e., 
\begin{equation}
p_\theta(X|H) = \prod_{i=1}^d \Pr(X_i|H) \label{eq:m1}
\end{equation}
and \begin{equation}
p_\theta(H|Y) = \prod_{i=1}^m \Pr(H_i|Y), \label{eq:m2}
\end{equation}
and are given by RBM-like conditional distributions, i.e.,
\begin{equation}
p_\theta(X_i=1|H) = \sigma\left(a_i + W_{i,\cdot}H\right) \label{eq:m3}
\end{equation}
and
\begin{equation}
p_\theta(H_i=1|Y) = \sigma\left(b_i + U_{i,\cdot}Y\right). \label{eq:m4}
\end{equation}
Hence the corresponding data generation probability is parametrized by 
$\theta = \left(\pi, a,b,W,U\right)$, where $\pi = \Pr(Y=1)$.

This data generation process is depicted in Figure~\ref{fig:YHX}.
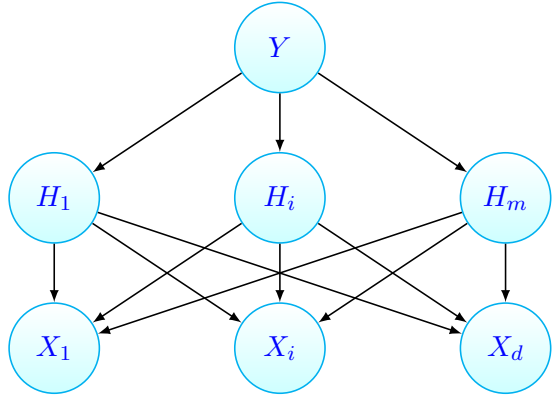
\begin{figure}[ht!]
  \centering
  \begin {tikzpicture}[-latex, auto ,node distance =4 cm and 5cm ,on grid ,
                semithick ,
                state/.style ={ circle ,top color =white , bottom color = processblue!20 ,
                        draw,processblue , text=blue , minimum width =1.2 cm}]
                \node[state] (Y) at (10,4) {$Y$};
                \node[state] (H1) at (7,2) {$H_1$};
                \node[state] (Hi) at (10,2) {$H_i$};
                \node[state] (Hm) at (13,2) {$H_m$};
                \node[state] (X1) at (7,0) {$X_1$};
                \node[state] (Xi) at (10,0) {$X_i$};
                \node[state] (Xd) at (13,0) {$X_d$};
                \path (H1) edge node [above =0.15 cm,left = 0.15cm] {}(X1);
                \path (H1) edge node [above =0.15 cm,left = 0.15cm] {}(Xi);
                \path (H1) edge node [above =0.15 cm,left = 0.15cm] {}(Xd);
                \path (Hi) edge node [above =0.15 cm,left = 0.15cm] {}(X1);
                \path (Hi) edge node [above =0.15 cm,left = 0.15cm] {}(Xi);
                \path (Hi) edge node [above =0.15 cm,left = 0.15cm] {}(Xd);             
                \path (Hm) edge node [above =0.15 cm,left = 0.15cm] {}(X1);
                \path (Hm) edge node [above =0.15 cm,left = 0.15cm] {}(Xi);
                \path (Hm) edge node [above =0.15 cm,left = 0.15cm] {}(Xd);
                \path (Y) edge node [above =0.15 cm,left = 0.15cm] {}(H1);
                \path (Y) edge node [above =0.15 cm,left = 0.15cm] {}(Hi);
                \path (Y) edge node [above =0.15 cm,left = 0.15cm] {}(Hm);      
        \end{tikzpicture}
  \caption{Data generated by a Markov Chain $Y \rightarrow H \rightarrow X$ with RBM-like conditional probabilities.}
  \label{fig:YHX}
 \end{figure}

The posterior probabilities $p_\theta(Y|X)$ are given by
\begin{align}
p_\theta(Y|X) & = \sum_{H \in \{0,1 \}^m}p_\theta(Y|H)p_\theta(H|X) \notag \\
& = \mathbb{E}_{h \sim p_\theta(H|X)}p_\theta(Y|H=h). \notag
\end{align}
By Section~\ref{sec:condIndCase}, we know that $p_\theta(H,Y)$ is equivalent to a RBM. Therefore, to accurately estimate the posterior, it suffices to approximate $p_\theta(H|X)$.

Under the data generation model described in Figure~\ref{fig:YHX} and equations~\eqref{eq:m1}-\eqref{eq:m4}, it is evident that  the joint distribution $p_\theta(X,H)$ cannot be parametrized as a RBM; indeed, $p_\theta(H|X)$ does not factorize. Hence, training a RBM on samples from $p_\theta(X)$, is a mean field approximation of $p_\theta(H|X)$.
The form of $p_\theta(X,H)$ is shown in the following lemma.
\begin{lemma} \label{lemma:mixture}
Under the data generation model described in Figure~\ref{fig:YHX} and equations~\eqref{eq:m1}-\eqref{eq:m4}, the joint distribution $p_\theta(X,H)$ is given by
\begin{equation}
p_\theta(X,H) = \exp\left(a^TX + X^TWH + b^TH\right)Z(H) \label{eq:rbmMix}\notag
\end{equation}
where
\begin{align}
&Z(H) = \frac{1}{\sum_{X \in \{0,1\}^d} \exp\left(a^TX  + X^TWH\right)}\notag\\
& \times \sum_{Y \in \{0,1\}}\frac{p_\theta(Y) \exp(H^TUY)}{\sum_{H'} \exp\left(b^TH' + H'^TUY\right)} \notag
\end{align}
\end{lemma}

\begin{proof}
By definition, 
\begin{align}
p_\theta(X,H) &= \sum_{Y \in \{0,1\}} p_\theta(X,H,Y) \notag \\
& = \sum_{Y \in \{0,1\}}p(Y)p_\theta(H|Y)p(X|H) \label{eq:mix}
\end{align}
Writing 
\begin{equation}
p_\theta(X|H) = \frac {\exp\left(a^TX + X^TWH\right)}{\sum_{X' \in \{0,1\}^d} \exp\left(a^TX' + X'^TWH\right)}\notag
\end{equation}
and similarly
\begin{equation}
p_\theta(H|Y) = \frac {\exp\left(b^TH + H^TUY\right)}{\sum_{H' \in \{0,1\}^m} \exp\left(b^TH' + H'^TUY\right)},\notag
\end{equation}
we obtain
\begin{align}
&p_\theta(X|H)p_\theta(H|Y) = \notag\\
&\frac{\exp\left(a^TX + X^TWH + b^TH + H^TUY\right)}{\left(\sum_{X'} \exp\left(aTX' + X'^TWH\right)\right)\left(\sum_{H'} \exp\left(b^TH' + H'^TUY\right)\right)}. \label{eq:complicated}
\end{align}
Plugging equation~\eqref{eq:complicated} in equation~\eqref{eq:mix} we get
\begin{align}
& p_\theta(X,H) = \exp\left(a^TX + X^TWH + b^TH\right)\notag\\
& \times \frac{1}{\sum_{X'} \exp\left(aTX' + X'^TWH\right)} \notag\\
& \times \sum_{Y \in \{0,1\}}\frac{p_\theta(Y) \exp(H^TUY)}{\sum_{H'} \exp\left(b^TH' + H'^TUY\right)} \notag
\end{align}
\end{proof}

From lemma~\ref{lemma:mixture} we see that $p_\theta(H|X)$ is close to be factorizable if $Z(H)$ is a approximately a log-linear function of $H$ and $p_\theta(X)$ is approximately a  log-linear function of $X$.


\section{Datasets used for our experiments}

\subsection{Simulated Dataset Generation Details}\label{app:datasets}

\begin{itemize}
\item {\bf CondInd}: the label $Y$ was sampled from a Bernoulli(0.5) distribution; The specificity $\eta_i$ and sensitivity  $\psi_i$ of the variables $X_i,\; i=1\ldots 5$ were sampled uniformly from $[0.5,1]$. The other ten $X_i$'s were random guesses, i.e., had specificity = sensitivity = $0.5$. 
\item {\bf Tree15-3-1}: the label $Y$ was sampled from a Bernoulli(0.5) distribution; each node in the intermediate and layer was generated from his parent with  specificity and sensitivity sampled uniformly from $[0.8,1]$, and in the bottom layer with  specificity and sensitivity sampled uniformly from $[0.6,1]$. 
\item {\bf LayeredGraph15-5-5-1}:  Data is generated from a Layered Graph with four layers of dimensions 1,5,5,15, starting at the true label $Y$. Each layer in the graph is generated from the above layer, and the graph has sparse connectivity (about 30\% of the edges exist). For every node $i$ and parent $j$ we sample specificity $\psi_{ij}$ and sensitivity $\eta_{ij}$ uniformly. Finally, the value at each node was calculated as the weighted sum of the probabilities of the node being 1 given the values of the nodes in the preceding layer, normalized by the sum over the edges. The label $Y$ was sampled from a Bernoulli(0.5) distribution.
\item {\bf TruncatedGaussian}: the label $Y$ was sampled from a Bernoulli(0.5) distribution. One Gaussian had mean vector $\mu_1$ were each of the 15 coordinates was sampled uniformly. The other Gaussian had mean vector $\mu_2 = -\mu_1$. Both Gaussians had identical covariance matrix, with off diagonal entries of $0.5$ and diagonal entries of $1$.

\end{itemize}

\begin{figure*}[ht!]
  \centering
  \includegraphics[width=2.5in,height=2in]{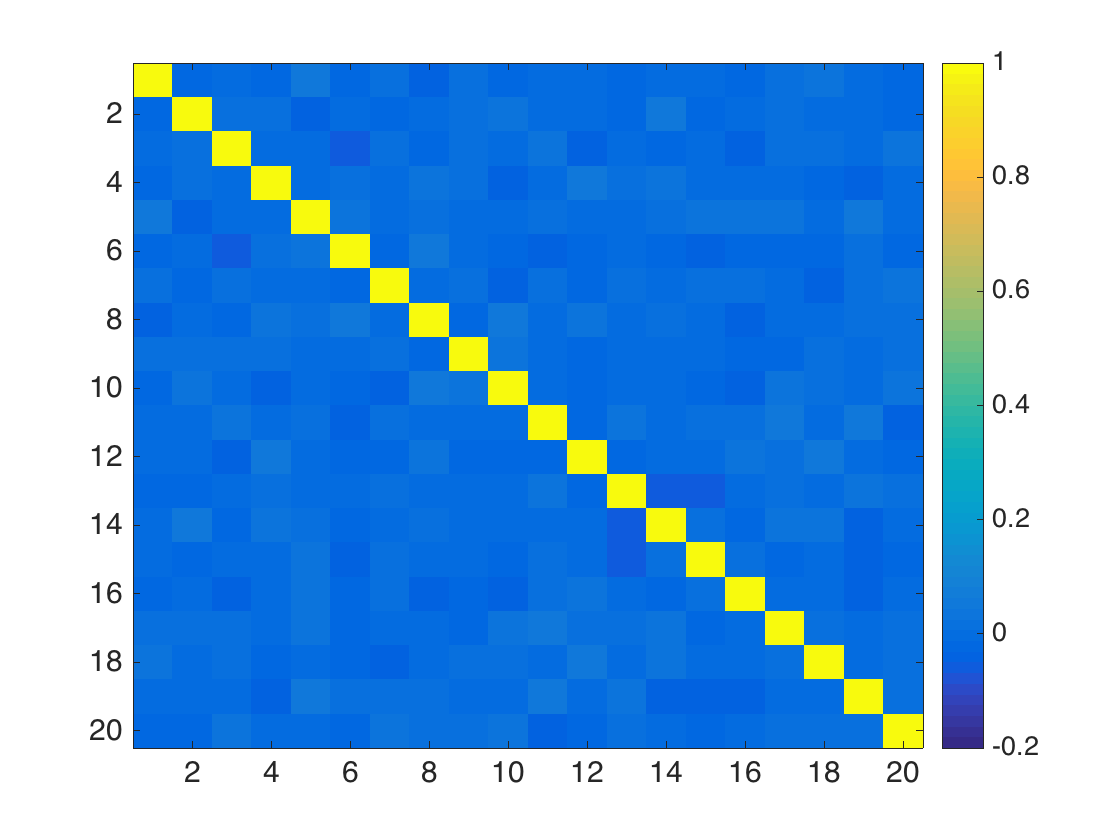}
  \includegraphics[width=2.5in,height=2in]{treeCondCorrelBefore.png}
  \includegraphics[width=2.5in,height=2in]{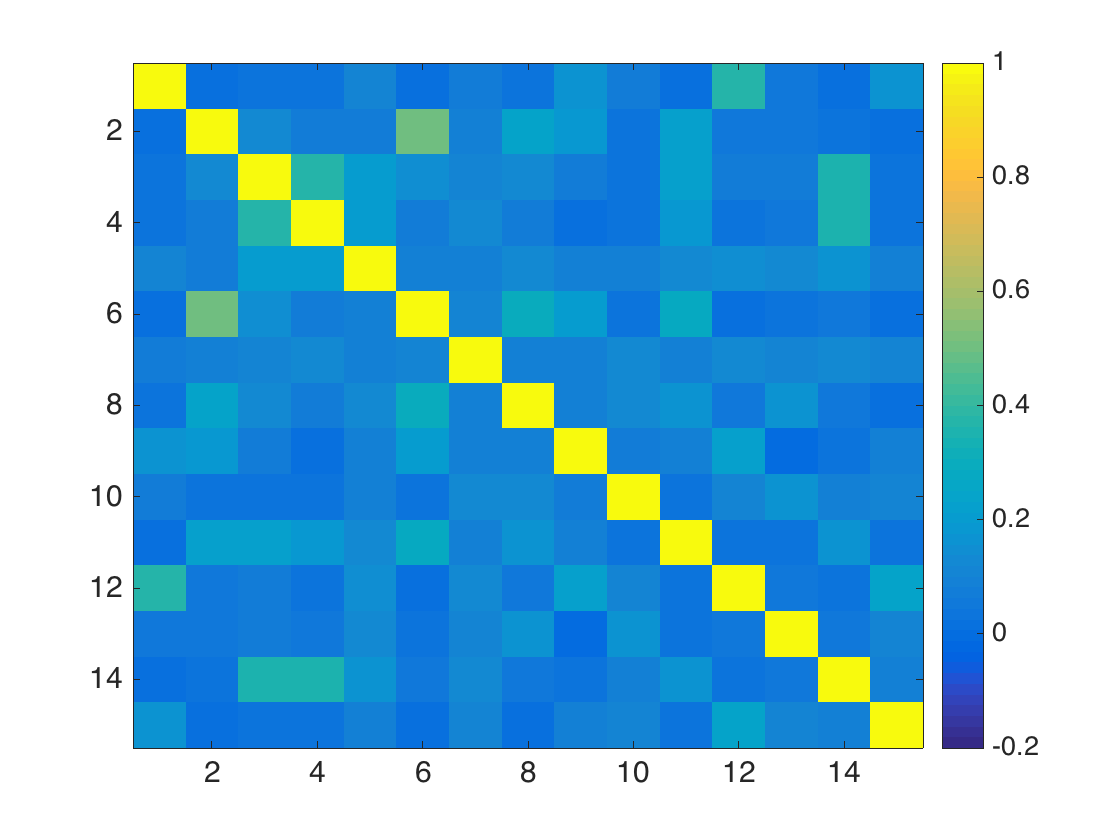}
  \includegraphics[width=2.5in,height=2in]{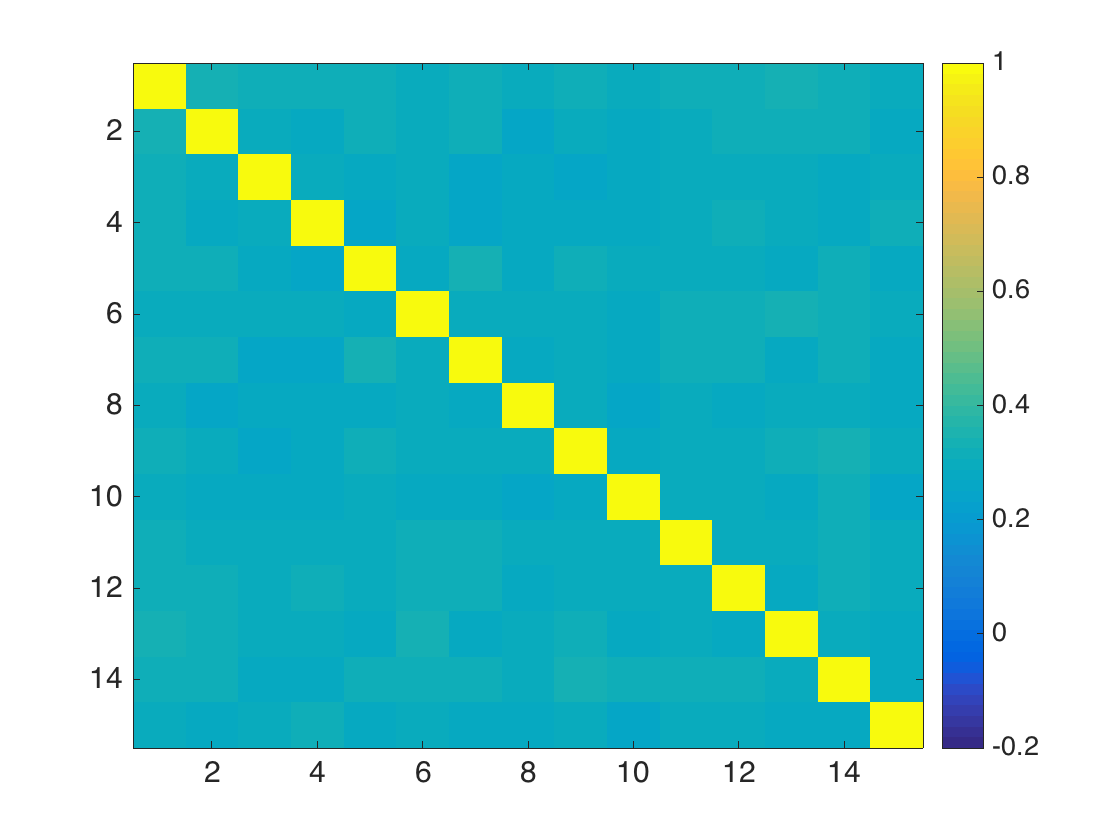}
  \caption{correlation matrices of the input data, for the $y=0$ class in all four simulated datasets: condInd (top left), tree15-3-1 (top right), LayeredGraph (bottom left), TruncatedGaussian (bottom right).}
  \label{fig:datasets}
 \end{figure*}
 

 \subsection{The Magic Datasets}\label{app:magic}
 An example for the correlation matrix of the 16 classifiers given the 0 class can be seen in Figure~\ref{fig:magic}.
 \begin{figure}
  \centering
  \includegraphics[width=3in,height=2in]{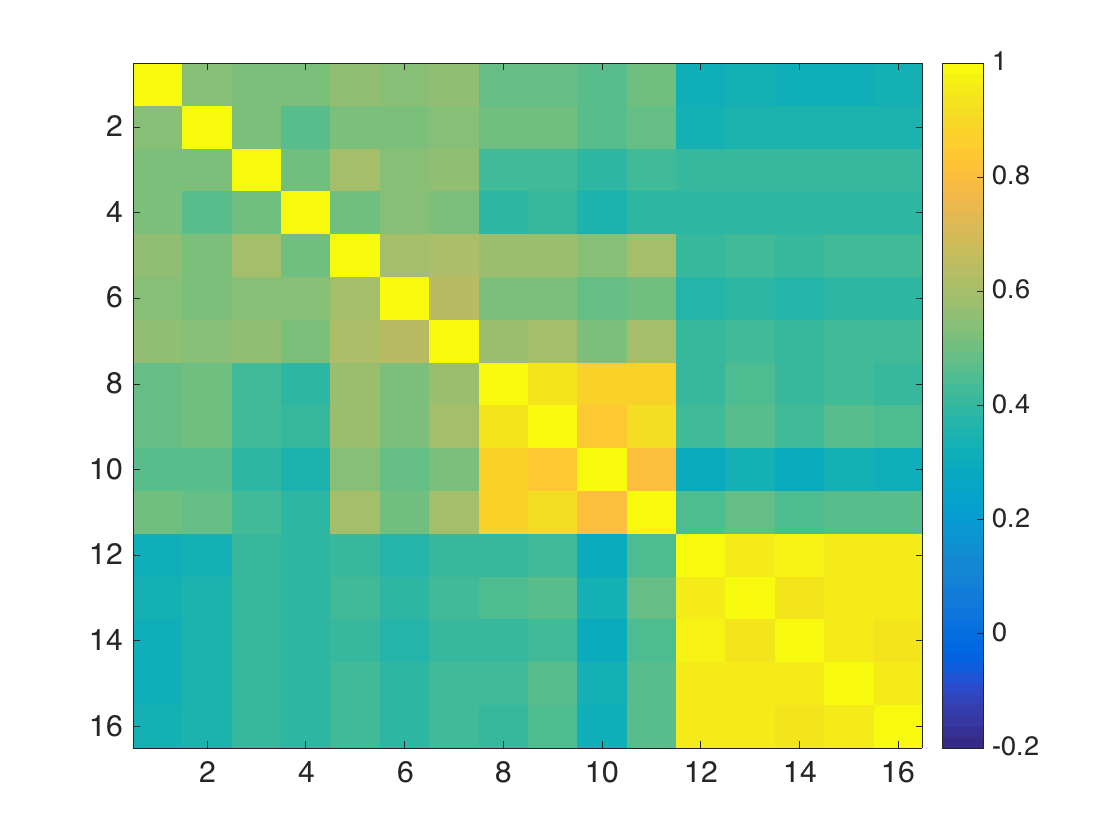}
  \caption{correlation matrix of the 16 classifiers in the Magic1 dataset, for the $y=0$ class.}
  \label{fig:magic}
 \end{figure}


\end{document}